\documentclass[11pt,a4paper]{article}
\usepackage{url}
\usepackage{amsmath}
\usepackage{mathtools}
\usepackage{amssymb}
\usepackage{algorithm}
\usepackage{algorithmicx}
\usepackage{authblk}
\usepackage[noend]{algpseudocode}
\usepackage{graphicx}

\DeclarePairedDelimiter\floor{\lfloor}{\rfloor}
\DeclareGraphicsExtensions{.png,.jpg}
\usepackage{color}
\usepackage{subfig}
\usepackage{bm}
\usepackage{amsthm}

\begin{document}
%
\title{A Theoretical Analysis of Deep Neural Networks for Texture Classification}


\author[1]{Saikat Basu}
\author[1]{Manohar Karki}
\author[2]{Robert DiBiano}
\author[1]{Supratik Mukhopadhyay}
\author[3]{Sangram Ganguly}
\author[4]{Ramakrishna Nemani}
\author[5]{Shreekant Gayaka}
\affil[1]{Department of Computer Science, Louisiana State University}
\affil[2]{Autopredictive Coding LLC}
\affil[3]{Bay Area Environmental Research Institute/ NASA Ames Research Center}
\affil[4]{NASA Advanced Supercomputing Division/ NASA Ames Research Center}
\affil[5]{Applied Materials}

\renewcommand\Authands{ and }


%


\maketitle

\newcommand{\fix}{\marginpar{FIX}}
\newcommand{\new}{\marginpar{NEW}}
\newtheorem{theorem}{Theorem}[section]
\newtheorem{proposition}{Proposition}[section]
\newtheorem{lemma}[theorem]{Lemma}

\begin{abstract}
 We investigate the use of Deep Neural Networks for the classification of image datasets where texture features are important for generating class-conditional discriminative representations. To this end, we first derive the size of the feature space for some standard textural features extracted from the input dataset and then use the theory of Vapnik-Chervonenkis dimension to show that hand-crafted feature extraction creates low-dimensional representations which help in reducing the overall excess error rate. As a corollary to this analysis, we derive for the first time upper bounds on the VC dimension of Convolutional Neural Network as well as Dropout and Dropconnect networks and the relation between excess error rate of Dropout and Dropconnect networks. The concept of \emph{intrinsic dimension} is used to validate the intuition that texture-based datasets are inherently higher dimensional as compared to handwritten digits or other object recognition datasets and hence more difficult to be shattered by neural networks. We then derive the mean distance from the centroid to the nearest and farthest sampling points in an n-dimensional manifold and show that the \emph{Relative Contrast} of the sample data vanishes as dimensionality of the underlying vector space tends to infinity.
\end{abstract}

\section{INTRODUCTION}
Texture is a key recipe for various object recognition tasks which involve texture-based imagery data like Brodatz~\cite{brodatz}, VisTex~\cite{vistex}, Drexel~\cite{drexel}, \\KTH~\cite{kth}, UIUCTex~\cite{lazebnik} as well as forest species datasets~\cite{forestSpecies2009}. Texture characterization has also been shown to be useful in addressing other object categorization problems like the Brazilian Forensic Letter Database (BFL)~\cite{bfl} which was later converted into a \emph{textural representation} in \cite{writerVerification}. In \cite{CIARP2013}, a similar approach was used to find a textural representation of the Latin Music Dataset \cite{LatinMusicDatabase2008}.   

Over the last decade, Deep Neural Networks have gained popularity due to their ability to learn data representations in both supervised and unsupervised settings and generalize to unseen data samples using hierarchical representations. A notable contribution in \emph{Deep Learning} is a \emph{Deep Belief Network}(DBN) formed by stacking \emph{Restricted Boltzmann Machines}~\cite{Hinton06afast}. Another closely related approach, which has gained much traction over the last decade, is the Convolutional Neural Network (CNN)~\cite{Lecun98gradient-basedlearning}. CNN's have been shown to outperform DBN in classical object recognition tasks like MNIST~\cite{mnist} and CIFAR~\cite{Krizhevsky09learningmultiple}. Despite these advances in the field of Deep Learning, there has been limited success in learning textural features using Deep Neural Networks. Does this mean that there is some inherent limitation in existing Neural Network architectures and learning algorithms? 

In this paper, we try to answer this question by investigating the use of Deep Neural Networks for the classification of texture datasets. First, we derive the size of the feature space for some standard textural features extracted from the input dataset. We then use the theory of Vapnik-Chervonenkis (VC) dimension to show that hand-crafted feature extraction creates low-dimensional representations, which help in reducing the overall excess error rate. As a corollary to this analysis we derive for the first time upper bounds on the VC dimension of Convolutional Neural Network as well as Dropout and Dropconnect networks and the relation between excess error rate of Dropout and Dropconnect networks. The concept of \emph{intrinsic dimension} is used to validate the intuition that texture-based datasets lie on an inherently higher dimensional manifold as compared to handwritten digits or other object recognition datasets and hence more difficult to be classified/shattered by neural networks. To highlight issues associated with the Curse of Dimensionality of texture datasets, we provide theoretical results on the mean distance from the centroid to the nearest and farthest sampling points in n-dimensional manifolds and show that the \emph{Relative Contrast} of the sample data vanishes as dimensionality of the underlying vector space tends to infinity. Our theoretical results and empirical analysis show that in order to classify texture datasets using Deep Neural Networks, we need to either integrate them with handcrafted features or devise novel neural architectures that can learn features from the input dataset that resemble these handcrafted texture features.  

\section{VC DIMENSION OF DEEP NEURAL NETWORKS AND CLASSIFICATION ACCURACY}\label{vc_dim_and_accuracy}

VC dimension was first proposed in \cite{vapnik:264} and was later applied to Neural Networks in \cite{Bartlett_vapnik-chervonenkisdimension}. It was noted in \cite{BianchiniS14} that the VC dimension proposed for Neural Networks is also applicable to Deep Neural Networks. It was shown in \cite{Bartlett_vapnik-chervonenkisdimension} that for neural nets with sigmoidal activation function, the VC-dimension is loosely upper-bounded by $O(w^4)$ where $w$ is the number of free parameters in the network. Given a classification model $M$, the VC-dimension of $M$ is the maximum number of samples that can be shattered by $M$.

We estimate the size of the sample space composed of the various features extracted from the  textural Co-occurrence Matrices (Haralick features) following those proposed in \cite{haralick1973}. We then use the theory of VC dimension to show that texture feature extraction creates low dimensional representations which help in reducing the overall excess error rate.     

\subsection[Sample complexity of Haralick features and the fat-shattering dimension]{Sample complexity of Haralick features and the \\ fat-shattering dimension\footnote{For a detailed description of the various GLCM metrics defined in this section and the notations used, we refer the reader to \cite{haralick1973}}}

 For the sake of simplicity, we consider intensity image with a single channel and Gray-Level Co-occurrence Matrix (GLCM) which can be easily extended to multi-channel images and Color Co-occurrence Matrices (CCM) without loss of generality. These features have also been shown to useful descriptors for aerial imagery datasets (\cite{basu2015deepsat}, \cite{basu2015semiautomated}). For $n{\times}n$ images with $\kappa$ color levels, the following results can be derived\footnote{A proof of these results follows from simple counting arguments.}.  

\begin{proposition}
If $x_1, x_2, \ldots x_{\kappa^2}$ be the values of the $\kappa{\times}\kappa$ GLCM matrices, then the number of distinct matrices is given by $\binom{n^2+\kappa^2-1}{\kappa^2-1}$.\qed
\end{proposition}

\begin{proposition}\label{GLCM_angular}
The number of distinct values for GLCM angular $2^{nd}$ moment is $n^4 - \Big({\floor*{\frac{n^2}{\kappa^2}}}^2{\times}\Big(\kappa^2-1\Big)+\Big(n^2-\Big(\kappa^2-1\Big)\floor*{\frac{n^2}{\kappa^2}}\Big)^2+1\Big)$\qed
\end{proposition}

\begin{proposition}
The number of distinct values of GLCM correlation is $n^2\kappa^2-n^2-\frac{\kappa^2}{2}+\frac{\kappa}{2}+1$.\qed
\end{proposition}

\begin{proposition}
The number of distinct values of GLCM sum average is $2n^2\kappa-2n^2+1$.\qed
\end{proposition}

\begin{proposition}\label{GLCM_contrast}
The number of distinct values of GLCM contrast is $n^2\kappa^2+n^2-2n^2\kappa+1$.\qed
\end{proposition}

From proposition \ref{GLCM_angular} through \ref{GLCM_contrast}, it can be seen that in the general case, number of distinct Haralick features is given by $O(n^2\kappa^2 + n^4)$. For deep neural networks, the VC dimension is upper bounded by $O(w^{4})$ according to \cite{Bartlett_vapnik-chervonenkisdimension}.
Now, we can pick the number of adjustable parameters $w$ to be such that $n{\leq}\kappa{\leq}w$ or $\kappa{\leq}n{\leq}w$. In both cases, we have $O(n^2\kappa^2){\leq}O(w^4)$ and $O(n^4){\leq}O(w^4)$ which gives $O(n^2\kappa^2+n^4){\leq}O(w^4)$.
Hence, the number of possible distinct values for the GLCM based feature vectors is much lower than the VC dimension of such a network. So, we can effectively argue that the VC-dimension of a Deep Neural Network with $w$ adjustable parameters is such that it can shatter the metrics formed using GLCM - the only prerequisite being that we select a network with the number of adjustable parameters as an upper bound for the input data dimensionality and the number of distinct gray levels in the color channel. On the other hand, in order to shatter the raw image vectors, the effective VC dimension of the network should be at least of the order of $O({n^{\kappa^2/4}})$. So, for the GLCM based features, we need Neural Networks with smaller VC dimension as compared to raw vectors. Also, in the next section, we show that with increase in VC dimension of the network, the excess error rate increases. So, the composite learning model formed by the integration of GLCM based features and Deep Neural Networks have lower excess error rate as compared to Deep Neural Networks combined with raw image pixels. 

\section{INPUT DATA DIMENSIONALITY AND \\ BOUNDS ON THE TEST ERROR} 

In this section, we derive the relation between input data dimensionality and upper bound $\Gamma$ on the excess error rate of the Deep Neural Network. This validates the fact that the lower dimensional representations of the Haralick feature space help in minimizing the test error rate. As a corollary to this analysis we derive for the first time upper bounds on the VC dimension of Convolutional Neural Network as well as Dropout and Dropconnect networks and show that the upper bound $\Gamma$ on the excess error rate of the Dropout networks is lower than that of DropConnect.  

\begin{lemma}\label{Lemma:data_model}
With increase in the dimensionality of the input data, the dimensionality of the optimal model increases.
\end{lemma}
\begin{proof}[Proof(Sketch)] As shown in \cite{Makhoul1989}, for input data dimensionality $d$ and model dimensionality $p$, the number of cells formed by $p$ planes in $d$ space is given by
\begin{equation}\label{Eq:C(p,d)}
C(p,d) = \sum\limits_{i=0}^{min(p,d)} {{p}\choose{i}} =
\begin{cases}
2^p
\hspace{3mm}   ,p \leq d  \\
\sum\limits_{i=0}^d {{p}\choose{i}}
\hspace{3mm} ,p>d.
\end{cases}
\end{equation}
Now, the number of cells per dimension gives the number of divisions of the model space along each dimension and can be approximated as $C(p,d)^{1/d}$. This in turn is equal to the number of class labels $c$. Therefore, for a given classification problem with $c$ class labels, we have, ${C(p,d)}^{1/d} = c$ and hence, we have
\begin{equation}\label{Eq:c_1/d}
c = {C(p,d)}^{1/d} =
\begin{cases}
2^{p/d}
\hspace{3mm}   ,p \leq d  \\
(\sum\limits_{i=0}^d {{p}\choose{i}})^{1/d}
\hspace{3mm} ,p>d.
\end{cases}
\end{equation}
From equation \ref{Eq:c_1/d}, it follows that with increase in data dimensionality $d$, the model dimensionality $p$ should increase, given a fixed classification problem with $c$ class labels. 
\end{proof}
\begin{lemma}\label{Lemma:data_vc}
With increase in the dimensionality of the model, its VC dimension increases.
\end{lemma}
\begin{proof}[Proof(Sketch)] This statement follows from the VC dimension bounds of both a Deep Neural Network and a Deep Convolutional Neural Network (CNN). The VC dimension of a Deep Neural Network is upper bounded by $O(w^4)$ and the VC dimension of a Convolutional Neural Network is upper bounded by $O\Big(\frac{m^4k^4s^{2l-2}}{l^2}\Big)$. The result for the Deep Neural Network follows from \cite{Bartlett_vapnik-chervonenkisdimension}, where it is noted that the VC dimension of Deep Neural Networks with sigmoidal activation functions is given by $O(t^2d^2)$ which reduces to $O(w^4)$. The result of the VC bound for the CNN along with the proof is detailed in Theorem \ref{Theorem:cnn_vc} below.  
\end{proof}

\begin{theorem}\label{Theorem:cnn_vc}
The VC dimension of a Convolutional Neural Network is upper bounded by $O\Big(\frac{m^4k^4s^{2l-2}}{l^2}\Big)$ where $m$ is the total number of maps, $k$ is the kernel size, $s$ is the subsampling factor and $l$ is the number of layers.
\end{theorem}
\begin{proof}[Proof(Sketch)]
From Theorem 5 and Theorem 8 in \cite{Bartlett_vapnik-chervonenkisdimension}, it can be seen that for the parameterized class $F = \{x \longmapsto f(\theta, x) : \theta \in {\mathbb{R}}^d\}$  with the arithmetic operations $+, -, \times, /$ and the exponential operation $\alpha \longmapsto e^{\alpha}$, jumps based on $>, \geq, <, \leq$, = and $\neq$ evaluations on real numbers and output 0/1, VCDim(F) = $O(t^2d^2)$. Here, $t$ is the number of operations and $d$ is the dimensionality of the adjustable parameter space. Now, for the CNN, input size is $n$, kernel size is $k$, sampling factor is $s$ and we assume convolution kernel step size as $1$ for simplicity. So, we have $\frac{\frac{\frac{n-k}{s}-k}{s}-k}{s} \ldots \text{upto $l$ layers} $ which in turn is equal to 1 for a binary classification problem\footnote{Note that for simplifying the algebra, we consider only the convolutional and subsampling layers of a CNN. This analysis can be extended to hybrid architectures with other types of layers (e.g., fully connected) by adjusting $t$ and $d$.}.
Now, in the simplest case, we have a CNN with one convolutional layer followed by one subsampling layer (c-s). Hence,$\frac{n-k}{s} = 1  \implies n = s + k$.
For a CNN with the configuration (c-s-c-s), we have,
\begin{equation}
\frac{\frac{n-k}{s}-k}{s} = 1  \implies n = k+s(s + k) = s^2 + ks + k
\end{equation}
Continuing this pattern, we have in the general case, 
\begin{equation}
\begin{split}
\frac{\frac{\frac{n-k}{s}-k}{s}-k}{s} \ldots \text{upto $l$ layers} = 1 \\
\implies n = s^l + ks^{l-1} + ks^{l-2} + \ldots + ks + k
\end{split}
\end{equation} 
Now, let $m_1, m_2, \ldots m_l$ be the number of maps in the various layers of a CNN and
$t = t_1 + t_2 + \ldots + t_l$ be the total number of operations.
Now, for layer 1, number of operations $t_1 = m_1(n-k)$, for layer 2, number of operations $t_2 = m_2(\frac{n-k}{s}-k)$, and so on. Therefore, Total number of operations
\begin{multline}
t = m_1(n-k) + \ldots + m_l(\frac{\frac{n-k}{s}-k}{s} \ldots \text{to $l$ layers})\\
= m_1(ks + ks^2 + \ldots + ks^{l-1} + s^l) + m_2(ks + ks^2 + \ldots \\
+ ks^{l-2} + s^{l-1}) + \ldots + m_ls
\end{multline}
Also, dimensionality of parameter space is given by $d = m_1k + m_2k + \ldots + m_lk $.
Now, for simplifying, if we assume that the number of maps in the layers $m_1 = m_2 = \ldots = m_l = \frac{m}{l}$, then, we have
\begin{multline}\label{Eq:operations}
t = \frac{m}{l}(n-k) + \frac{m}{l}(\frac{n-k}{s}-k) + \ldots + \frac{m}{l}(\frac{\frac{\frac{n-k}{s}-k}{s}-k}{s} \\ \ldots \text{upto $l$ layers})
= \frac{mks^2(s^{l-1}-1)}{l(s-1)^2} + \frac{ms(s^l-1)}{l(s-1)} \\= O(\frac{mks^{l-1}}{l})
\end{multline}
\begin{equation}\label{Eq:dim}
\text{Also, }\hfill d = O(mk)
\end{equation}
From equation \ref{Eq:operations} and equation \ref{Eq:dim}, we have VCdim$_{CNN}$ = $O\Big(\frac{m^4k^4s^{2l-2}}{l^2}\Big)$

\end{proof}

\begin{theorem}\label{Theorem:error_vc}
Upper bound on excess error rate $\mathcal{E}$ increases with increase in VC dimension given fixed number of training samples $N$.
\end{theorem}
\begin{proof}[Proof(Sketch)] According to the theory of VC dimension \cite{vapnik1996structure}, we have Excess error rate

\begin{equation}\label{Eq:excess_error}
\mathcal{E} \leq \sqrt{\frac{h(\log(2N/h)+1)-\log(\eta/4)}{N}} 
\end{equation}

where, $h$ is the VC dimension of the model, $N$ is the number of training samples and $0\leq\eta\leq1$. From equation \ref{Eq:excess_error}, the result follows.  
\end{proof}

\begin{theorem}\label{Theorem:vc_dropout}
For a given Dropout network with probability of dropout $p$ and number of adjustable paramaters in the network being $w$, the VC dimension of the network is upper bounded by $O\Big((1-p)^8w^4\Big)$.
\end{theorem}
\begin{proof}[Proof(Sketch)] For a neural network with number of neurons $n = n_1 + n_2 + n_3 + \ldots + n_l$, the number of adjustable paramaters $w$ is given by
\begin{equation}
w = n_1n_2 + n_2n_3 + n_3n_4 + \ldots + n_{l-1}n_l
\end{equation}
For a given dropout fraction $p$, each neuron in the network can be dropped by a probability of $p$. So the effective number of neurons in the Dropout network
\begin{equation}
\widetilde{n} = (1-p)(n_1 + n_2 + \ldots + n_l)
\end{equation} 
Now, we can split the effective number of neurons in each layer as
$\widetilde{n_1} = (1-p)n_1 \text{,   } \ldots \text{,   } \widetilde{n_l} = (1-p)n_l$.
\begin{multline}
\text{Therefore, } \widetilde{w} = \widetilde{n_1} \widetilde{n_2} + \widetilde{n_2} \widetilde{n_3} + \ldots + \widetilde{n_{l-1}} \widetilde{n_{l}} \\
= (1-p)^2(n_1 n_2 + \ldots + n_{l-1} n_{l}) = (1-p)^2w
\end{multline} 
Now, given that $\text{VCDim}_{Dropout} = O\Big({\widetilde{w}}^4\Big)$ we have,
$ {\widetilde{w}}^4 =  ((1-p)^2)^4w^4 = O\Big((1-p)^8w^4\Big)$.
So,
\begin{equation}\label{Eq:vc_dropout}
\text{VCDim}_{Dropout} = O\Big((1-p)^8w^4\Big)
\end{equation} 
\end{proof}

\begin{theorem}\label{Theorem:vc_dropconnect}
For a given Dropconnect network with probability of drop $p$ and number of adjustable paramaters in the network being $w$, the VC dimension of the network is upper bounded by $O\Big((1-p)^4w^4\Big)$.
\end{theorem}
\begin{proof}[Proof(Sketch)] Since in a Dropconnect network, each weight can be dropped by a probability of $p$, so, effective number of adjustable parameters in the Dropconnect network is given by $\widetilde{w} = (1-p)w$.
Now, given that,
${\widetilde{w}}^4 =  (1-p)^4w^4 = O\Big((1-p)^4w^4\Big)$, we have,
\begin{equation}\label{Eq:vc_dropconnect}
\text{VCDim}_{Dropconnect} = O\Big((1-p)^4w^4\Big)
\end{equation} 
\end{proof}

\begin{theorem}\label{Theorem:dropout_dropconnect}
For a given drop probability $p$, the number of adjustable paramaters in the network being $w$, the excess error rate being $\mathcal{E}$ and the upper bounds on the error rates of the Dropout and Dropconnect networks being ${\Gamma}_{Dropout}$ and ${\Gamma}_{Dropconnect}$ respectively, we have ${\Gamma}_{Dropout} \leq {\Gamma}_{Dropconnect}$.
\end{theorem}
\begin{proof}[Proof(Sketch)] From Equation \ref{Eq:excess_error} and \ref{Eq:vc_dropout}, we have
\begin{equation}
\text{Excess error rate}, \mathcal{E} \leq \sqrt{\frac{h(\log(2N/h)+1)-\log(\eta/4)}{N}}
\end{equation}
Therefore, upper bound on the error rate ${\Gamma}_{Dropout} $
\begin{equation}
= \sqrt{\frac{1}{N}(1-p)^8w^4\Big[\log\Big(\frac{2N}{(1-p)^8w^4}+1\Big)\Big] - \log(\eta/4) }
\end{equation}
Similarly, from Equation \ref{Eq:excess_error} and \ref{Eq:vc_dropconnect}, we have for the Dropconnect network,
 ${\Gamma}_{Dropconnect}$
\begin{equation}
 = \sqrt{\frac{1}{N}(1-p)^4w^4\Big[\log\Big(\frac{2N}{(1-p)^4w^4}+1\Big)\Big] - \log(\eta/4) }
\end{equation}
For a given $w$, $N$ and probability of drop $p$ with $0 \leq p < 1$, it can be easily shown that the upper bounds on the excess error rates of the Dropout and Dropconnect networks are related as
\begin{equation}
{\Gamma}_{Dropout} \leq {\Gamma}_{Dropconnect}
\end{equation}
\end{proof}
This is substantiated by the experimental results in Section \ref{experiments}.

\section{WHAT IS THE DIFFERENCE BETWEEN OBJECT RECOGNITION DATASETS AND \\ TEXTURE-BASED DATASETS IN TERMS OF DIMENSIONALITY?}\label{section:object_recognition_vs_texture}  
We argue that object recognition datasets lie on a much lower dimensional manifold than texture datasets. Hence, even if Deep Neural Networks can effectively shatter the raw feature space of object recognition datasets, the dimensionality of texture datasets is such that without explicit texture-feature extraction, these networks cannot shatter them. In order to estimate the dimensionality of the datasets, we use the concept of \emph{intrinsic dimension}\cite{MLE04}. 

\subsection{Intrinsic Dimension Estimation using the Maximum Likelihood algorithm}
The \emph{intrinsic dimension} of a dataset represents the minimum number of variables that are required to represent the data. We use the Maximum Likelihood algorithm proposed in \cite{MLE04} to estimate the Intrinsic dimension of various datasets. The results for the various datasets and the Haralick features extracted are listed in Table \ref{table:Intrinsic_Dimension_1} and  \ref{table:Intrinsic_Dimension_2}. The DET dataset~\cite{ILSVRC15} is a subset of the Imagenet dataset.

\begin{table}[h]
\centering
\begin{tabular}{ | c | c | c | c | }
    \hline
 \textbf{Dataset}  &MNIST& CIFAR10 & DET \\  \hline
  \textbf{Intrinsic Dim.} & 9.96 & 15.9  &  17.01 \\ \hline 
  \end{tabular}
  \caption{Intrinsic Dimension estimation using MLE on the MNIST, CIFAR-10  and DET datasets}
  \label{table:Intrinsic_Dimension_1}
\end{table}

\begin{table*}[ht!]
\centering
\begin{tabular}{ | c | c | c | c | c | c | c | c |}
    \hline
 \textbf{Dataset}  & \textbf{Brodatz} & \textbf{VisTex} & \textbf{KTH}  \\  \hline
  \textbf{Intrinsic Dimension (Raw Vect.)} & 34.87 & 44.81 & 43.69 \\ \hline
  \textbf{Intrinsic Dimension (Texture)} & 4.03 & 3.84 & 3.73 \\ \hline
  \textbf{Dataset}  & \textbf{KTH2} & \textbf{Drexel} & \textbf{UIUCTex} \\  \hline
  \textbf{Intrinsic Dimension (Raw Vect.)} & 54.19 & 30.26 & 33.64\\ \hline
  \textbf{Intrinsic Dimension (Texture)} & 3.93 & 4.24 & 4.57\\ \hline
  \end{tabular}
  \caption{Intrinsic Dimension estimation using MLE on the 6 texture datasets}
  \label{table:Intrinsic_Dimension_2}
\end{table*}

From Table \ref{table:Intrinsic_Dimension_1} and \ref{table:Intrinsic_Dimension_2}, we can see that the intrinsic dimensionality of the texture datasets (Brodatz, VisTex, KTH, KTH2, Drexel and UIUCTex) is much higher than that of object recognition datasets (MNIST, CIFAR-10 and DET). So, without explicit texture-feature extraction, a deep neural network cannot shatter the texture datasets because of their intrinsically high dimensionality. However, as seen in Table \ref{table:Intrinsic_Dimension_2}, the features extracted from the texture datasets have a much lower intrinsic dimensionality and lie on a much lower dimensional manifold than the raw vectors and hence can be shattered/classified even by networks with relatively smaller architectures. Once, we have validated the fact that texture-based datasets lie on a higher dimensional manifold as compared to handwritten digit or object recognition datasets, we highlight issues associated with the high dimensionality of texture datasets. 

\section{CURSE OF DIMENSIONALITY IN TEXTURE DATASETS}
\emph{Curse of Dimensionality} refers to the phenomenon where classification power of the model decreases with increase in dimensionality of the input feature space. In the following sections, we derive some theoretical results on \emph{Curse of Dimensionality} for high-dimensional texture data.

\subsection{Sampling data in Higher Dimensional Manifolds}

The mean distance from the centroid to the nearest sampling point is a useful metric for quantifying the hardness of classification \cite{hastie01statisticallearning}. To compute this mean distance, we first state a result on computing the expected value of a non-negative random variable and then use it to compute the mean distance from the centroid to the nearest sample point. The median distance was computed in \cite{hastie01statisticallearning}. However, to get a more accurate estimate of the distance metrics, we compute the mean in this paper.    

\begin{lemma}\label{Lemma:expected_val}
If a random variable $y$ can take on only non-negative values, then the mean or expected value of $y$ is given by $\int_0^\infty [1-F_x(t)]dt$.
\end{lemma}
\begin{proof}[Proof(Sketch)] Since $1-F_X(x) = P(X \geq x) = \int_x^{\infty}f_X(t) dt$, it follows that
$\int_0^\infty(1-F_X(x)dx) = \int_0^{\infty}P(X \geq x)dx = \int_0^{\infty}\int_x^{\infty}f_X(t)dtdx$.
Changing the order of integration, we have
$\int_0^\infty(1-F_X(x)dx) = \int_0^{\infty} \int_0^t f_X(t)dxdt = \int_0^{\infty}x[f_X(t)]_0^t dt = \int_0^{\infty}tf_X(t)dt$.
Now, taking the substitution $t=x$ and $dt=dx$, the expected value\\
\begin{equation}
E(X) = \int_0^\infty(1-F_X(x)dx) = \int_0^{\infty} (1-y^p)^n dy
\end{equation}
\end{proof}

\begin{lemma}\label{Lemma:mean_seperation}
Consider $n$ samples distributed uniformly in a $p$-dimensional hypersphere of radius 1 and center at (0,0). If at the origin, we consider a nearest neighbor estimate, then the mean distance from the origin to the nearest sampling point is $\prod_{\xi=1}^n (1 + \frac{1}{p\xi})^{-1}$.
\end{lemma}

\begin{proof}[Proof(Sketch)] For a ball of radius $r$ in $R^p$ the volume is given by $\omega_pr^p$, where $\omega_p$ is denoted as $\frac{\pi^{p/2}}{(p/2)!}$
So, the probability of a point sampled uniformly from the unit ball lying within a distance $x$ of the origin is the ratio of the volume of that ball to the volume of the unit ball. The common factors of ${\omega}_p$ cancel, so we get the Cumulative Distribution Function (CDF) and Probability Density Function (PDF) as 
$F(x)=x^p, \text{and} f(x)=px^{p-1},  0 \leq x \leq 1$.
From \cite{Hogg2005}, for $n$ points with CDF $F$ and PDF $f$, we have the following general formula for the $\xi^{th}$ order statistic
\begin{equation}\label{Eq:hogg}
g_k(y_\xi) = \frac{n!}{(\xi-1)!(n-\xi)!}[F(y_\xi)]^{\xi-1}[1-F(y_\xi)]^{n-\xi}f(y_\xi)
\end{equation}
So, we have the minimum by setting $\xi=1$ as
\begin{equation}
g(y)=n(1-F(y))^{(n-1)}f(y) = n(1-y^p)^{n-1}py^{p-1}
\end{equation}
This yields the CDF, $G(y)=1-(1-y^p)^n$.
The random variable $y$ can take on only non-negative values. So, by Lemma \ref{Lemma:expected_val} the mean or expected value is $E[X]=\int_0^\infty [1-G_x(t)]dt$.
Now, by substituting $x^p$ by $z$, we have 
$E(X) = \frac{1}{p} \int_0^1 z^{\frac{1}{p}-1}{(1-z)}^n dz $. (Note the change of 
limits since z lies in [0,1]).

This can be reduced using the Euler Gamma function as $E(X) = \frac{1}{p}.\frac{\Gamma(\frac{1}{p})\Gamma(n+1)}{\Gamma(n+1+\frac{1}{p})}$.
Now, by using the identity $\Gamma(z+1) = z\Gamma(z)$ recursively, we get Mean Distance,
\begin{equation}
D(p,N) = E(X) = \prod_{\xi=1}^n (1 + \frac{1}{p\xi})^{-1}
\end{equation}
\end{proof}

\begin{table*}[ht!]
\centering
\begin{tabular}{ | c | c | c | c | c | c | c | c |c | c | c | c |}
    \hline
 \textbf{Dataset} & \textbf{MNIST} & \textbf{CIFAR-10} & \textbf{DET} & \textbf{Brodatz} & \textbf{VisTex} \\  \hline
  \textbf{D(p,N)}  & 0.32 & 0.49 & 0.54 & 0.74 & 0.79 \\  \hline
  \end{tabular}
 
 \begin{tabular}{ | c | c | c | c | c | c | c | c |c | c | c | c |}
    \hline
 \textbf{Dataset} & \textbf{KTH} & \textbf{KTH2} & \textbf{Drexel} & \textbf{UIUCTex}\\  \hline
  \textbf{D(p,N)} & 0.78 & 0.79 & 0.63 & 0.69\\  \hline
  \end{tabular}
   
  \caption{Mean distance from origin to nearest sampling point for various object recognition and texture datasets}
  \label{table:mean_distance_to_closest_point}
\end{table*}

Table \ref{table:mean_distance_to_closest_point} shows the mean distance from the origin to the nearest sampling point for various datasets. From the table and according to \cite{hastie01statisticallearning}, most data points for the texture datasets are nearer to the feature space boundary than to any other data point. This makes prediction particularly difficult for these datasets because we cannot interpolate between data points and we need to extrapolate. Next, we propose a result on the expected distance from the origin to the farthest data point and then use it to derive the relation of the \emph{Relative Contrast} of the data points to the underlying dimensionality of the vector space as highlighted in Section \ref{Sec:relative_contrast}.

\begin{lemma}\label{Lemma:mean_max_seperation}
Consider $n$ samples distributed uniformly in a $p$-dimensional hypersphere of radius 1 and center at (0,0). If at the origin, we consider a nearest neighbor estimate, then mean distance from origin to the farthest data point is $1-\frac{np}{(np+p-1)(np+p)}$.
\end{lemma}

\begin{proof}[Proof(Sketch)] Using equation \ref{Eq:hogg}, and setting $\xi=n$ for the maxima, we have,
\begin{equation}
g(y) = n[F(y)]^{n-1}f(y) = ny^{pn-1}py^{p-1} = npy^{pn+p-2}
\end{equation} 
Therefore, the corresponding CDF is given by $G(y) = np\frac{y^{pn+p-1}}{pn+p-1}$.
By Lemma \ref{Lemma:expected_val}, the mean or expected value is $E[X] = \int_0^{\infty} [1-np\frac{y^{pn+p-1}}{pn+p-1}]dy$
\begin{equation}
= 1 - \frac{np}{(np+p-1)(np+p)}
\end{equation}
\end{proof}

\begin{table*}[ht!]
\centering
\begin{tabular}{ | c | c | c | c  |}
    \hline
     \textbf{Texture Datasets}  & \textbf{Brodatz} & \textbf{Drexel} & \textbf{KTH}\\  \hline
    CNN Test Error (\%) & 28.96 & 35.27 & 34.93 \\ \hline
    \textbf{Texture Datasets}  & \textbf{KTH2} & \textbf{UIUCTex} & \textbf{VisTex}\\  \hline
    CNN Test Error (\%) & 40.29  &  49.75 &  26.68 \\ \hline
  \end{tabular}
  \caption{Test Error of a Convolutional Neural Network trained using supervised backpropagation on the various texture datasets.}
  \label{table:CNN_accuracy_texture}
\end{table*} 

\subsection{Relative Contrast in High Dimensions}\label{Sec:relative_contrast}

In \cite{Beyer1999}, it was shown that as dimensionality increases, the distance to the nearest neighbor approaches that of the farthest neighbor, i.e., contrast between points vanishes, while, in \cite{Aggarwal01} it was shown that \emph{Relative Contrast} varies as $\sqrt{p}$ for $n=2$ sample points with dimensionality $p$. In this paper, we generalize this to the case of $n$ data points and also provide an exact estimate of the \emph{Relative Contrast} instead of providing approximation bounds as \cite{Aggarwal01}. We then show that as dimensionality $p \to \infty$, it yields the same result as \cite{Beyer1999} and \cite{Aggarwal01}. Also, we eliminate the arbitrary constant $C$ used in \cite{Aggarwal01} which can vary significantly with change in parameters resulting in a fluctuating bound. It should be noted that we assume the $L_2$ norm distance metric and the Euclidean space for deriving our algebra.

\begin{theorem}\label{Theorem:RC}
If ${RC}_{n,p}$ be the Relative Contrast of $n$ uniformly distributed sample points with $p$ being the dimensionality of the underlying vector space, then,
${RC}_{n,p} = \frac{1-\frac{np}{(np+p-1)(np+p)}-\prod_{\xi=1}^n (1 + \frac{1}{p\xi})^{-1}}{\prod_{\xi=1}^n (1 + \frac{1}{p\xi})^{-1}}$ and ${RC}_{n,p}$  approaches 0 as p approaches $\infty$.
\end{theorem}      
\begin{proof}[Proof(Sketch)] From Lemma \ref{Lemma:mean_seperation}, we can see that the mean distance from the origin to the nearest sampling point is given by the expression $\prod_{\xi=1}^n (1 + \frac{1}{p\xi})^{-1}$.
And from Lemma \ref{Lemma:mean_max_seperation}, the mean distance from the origin to the farthest data point is given by the expression $1 - \frac{np}{(np+p-1)(np+p)}$.
Therefore, $\frac{E[Dmax-Dmin]}{E[Dmin]}$
\begin{equation}\label{Eq:RC}
= \frac{1-\frac{np}{(np+p-1)(np+p)}-\prod_{\xi=1}^n (1 + \frac{1}{p\xi})^{-1}}{\prod_{\xi=1}^n (1 + \frac{1}{p\xi})^{-1}}
\end{equation}
Now, it can be easily shown that
\begin{equation}
\lim_{p \to \infty} \frac{1-\frac{np}{(np+p-1)(np+p)}-\prod_{\xi=1}^n (1 + \frac{1}{p\xi})^{-1}}{\prod_{\xi=1}^n (1 + \frac{1}{p\xi})^{-1}} = 0
\end{equation}
Therefore, it follows that, $\frac{E[Dmax-Dmin]}{E[Dmin]} \to 0 ~\text{as}~ p \to \infty$.
Therefore, ${RC}_{n,p} \to 0 ~\text{as}~ p \to \infty $.
From equation \ref{Eq:RC}, it can be concluded that for the general case of $n$ sample points with a dimensionality of $p$, the expected value of the \emph{relative contrast} for the sample points varies as $p^{-(n+1)}$.
\end{proof}

\begin{theorem}\label{Theorem:RC_approximation}
For $n=2$ the general result proposed in Theorem \ref{Theorem:RC} approaches the bound of $\frac{C}{\sqrt{p}}\sqrt{\frac{1}{2\xi+1}}$ proposed in \cite{Aggarwal01} as the dimensionality $p$ of the underlying sample space approaches $\infty$.
\end{theorem}
\begin{proof}[Proof(Sketch)] From \cite{Aggarwal01}, it can be seen that for dimensionality of $p$ and $L_k$ norm, 
\begin{equation}\label{Eq:Aggarwal}
\lim_{p \to \infty} E[\frac{Dmax - Dmin}{Dmin}.\sqrt{p}] = C\sqrt{\frac{1}{2\xi+1}}
\end{equation}
Subtracting the rightmost term in Equation \ref{Eq:RC} from Equation \ref{Eq:Aggarwal}, we have, ${RC}_{diff} = {RC}_{Agg} - {RC}_{Ours}$
\begin{equation}
\begin{split}
= \frac{C}{\sqrt{p}}\sqrt{\frac{1}{2\xi+1}}
- \frac{1-\frac{np}{(np+p-1)(np+p)}-\prod_{\xi=1}^n (1 + \frac{1}{p\xi})^{-1}}{\prod_{\xi=1}^n (1 + \frac{1}{p\xi})^{-1}}
\end{split}
\end{equation}
Therefore, for any arbitrary constant $C$ and a given $\xi$,
\begin{multline}\label{Eq:RC_diff}
lim_{p \to \infty}{RC}_{diff} = \lim_{p \to \infty} \Big(\frac{C}{\sqrt{p}}\sqrt{\frac{1}{2\xi+1}} \\
- \frac{1-\frac{np}{(np+p-1)(np+p)}-\prod_{\xi=1}^n (1 + \frac{1}{p\xi})^{-1}}{\prod_{\xi=1}^n (1 + \frac{1}{p\xi})^{-1}}\Big) 
\end{multline}
Therefore, by substituting $n=2$ in Equation \ref{Eq:RC_diff}, it is easy to show that $
lim_{p \to \infty}{RC}_{diff} = 0$.
\end{proof}
Theorem \ref{Theorem:RC} validates the result in \cite{Beyer1999} and Theorem \ref{Theorem:RC_approximation} shows that for the special case $n=2$, our result approaches the bound of \cite{Aggarwal01} as dimensionality $p$ approaches $\infty$. So, from Section \ref{section:object_recognition_vs_texture} and Theorem \ref{Theorem:RC}, we conclude that texture datasets lie on an inherently higher dimensional manifold than object recognition datasets, so their Relative Contrast is lower.
 
\section{EXPERIMENTS}\label{experiments}

To validate our theory that error rate for networks with Haralick features is lower than that of raw vectors, we performed experiments on 6 benchmark texture classification datasets - Brodatz, VisTex, Drexel, KTH-TIPS, KTH-TIPS2 and UIUCTex. We extracted 27 features based on the GLCM metrics presented in Section \ref{vc_dim_and_accuracy}. Without loss of generality, we select image size $n$\footnote{Note that we extract $n{\times}n$ sliding window blocks from the various texture datasets for uniformity of analysis.} to be 28 and number of color levels $\kappa$ as 256. Also, datasets with multiple color channels are converted to grayscale. The Deep Neural Networks are trained by stacking -- 1) Restricted Boltzmann Machines (RBM) and 2) Denoising Autoencoders (SDAE). Both the models are then discriminatively fine-tuned with supervised backpropagation. Figures \ref{experiments_DBN} and \ref{experiments_SAE} show the final test error of the backpropagation algorithm on the labeled test data using RBM and SDAE for unsupervised pre-training. Table \ref{table:CNN_accuracy_texture} shows the final test error on the various texture datasets using a CNN. Our CNN has 3 convolutional layers with 32, 32 and 64 feature maps with $5{\times}5$ kernels each accompanied with max-pooling layers with $3{\times}3$ kernels. Each pooling layer is followed by a layer with Rectified Linear units and a local response normalization layer with a $3{\times}3$ locality. We use a softmax based loss function and a learning rate which is initially set to 0.001 and then decreased as the inverse power of a gamma parameter (0.0001). In \cite{hafemann2014analysis}, the authors proposed a new CNN architecture for texture classification. However, in this paper, we focus on the CNN architecture proposed in \cite{Lecun98gradient-basedlearning} to maintain uniformity with our theoretical analysis. By comparing the results in Figure \ref{experiments_DBN}, \ref{experiments_SAE} and Table \ref{table:CNN_accuracy_texture}, we can see that for all texture datasets, Haralick feature based networks outperform the networks based on raw pixels. So, the experiments substantiate our theoretical claim that extraction of Haralick features create low-dimensional representations that enable Deep Neural Networks to achieve lower test error rate.

\begin{figure*}[h!]
  \centering
  \subfloat[Brodatz]{\label{figur:1}\includegraphics[width=0.3\textwidth]{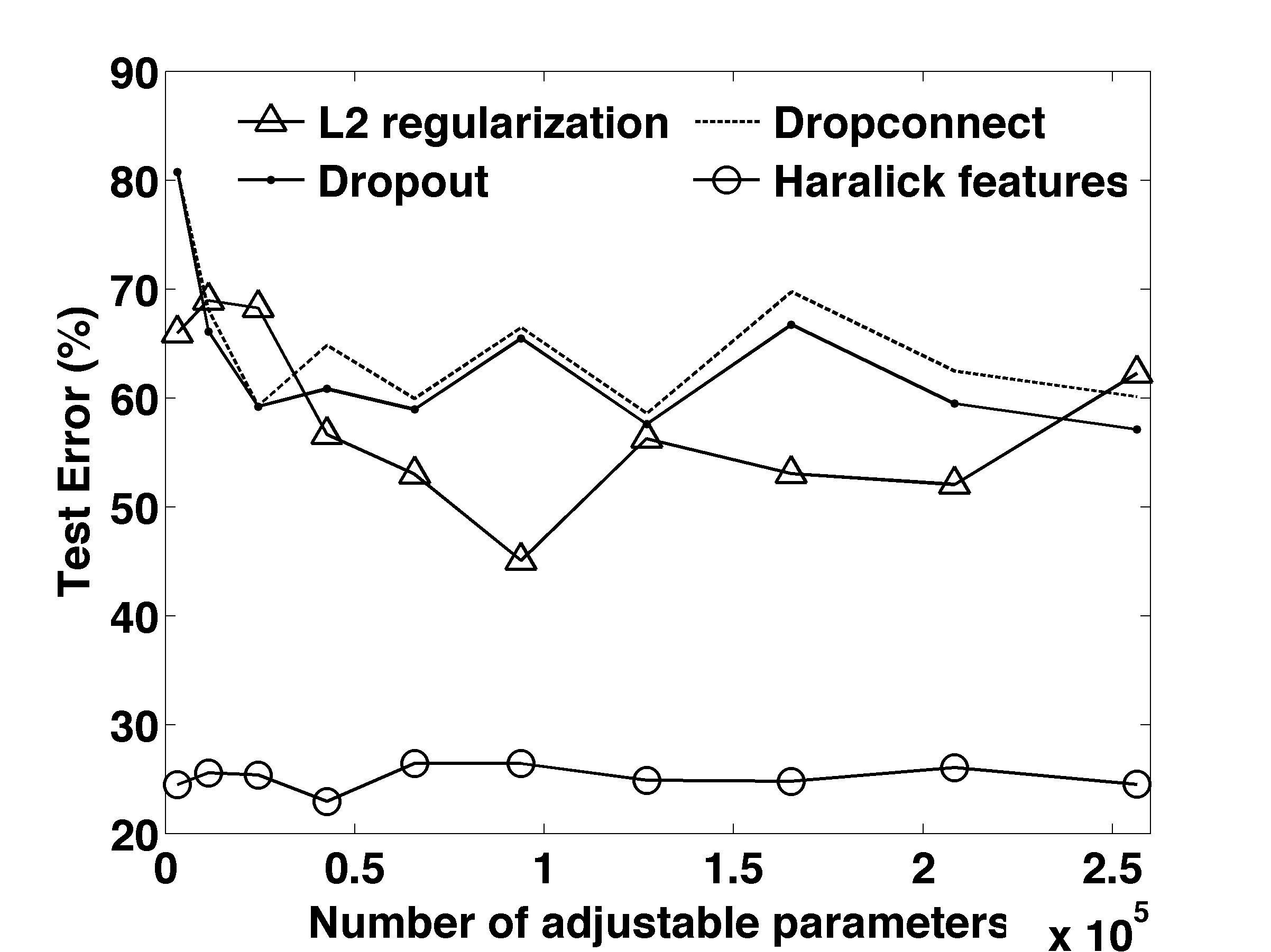}}
  \subfloat[Drexel]{\label{figur:2}\includegraphics[width=0.3\textwidth]{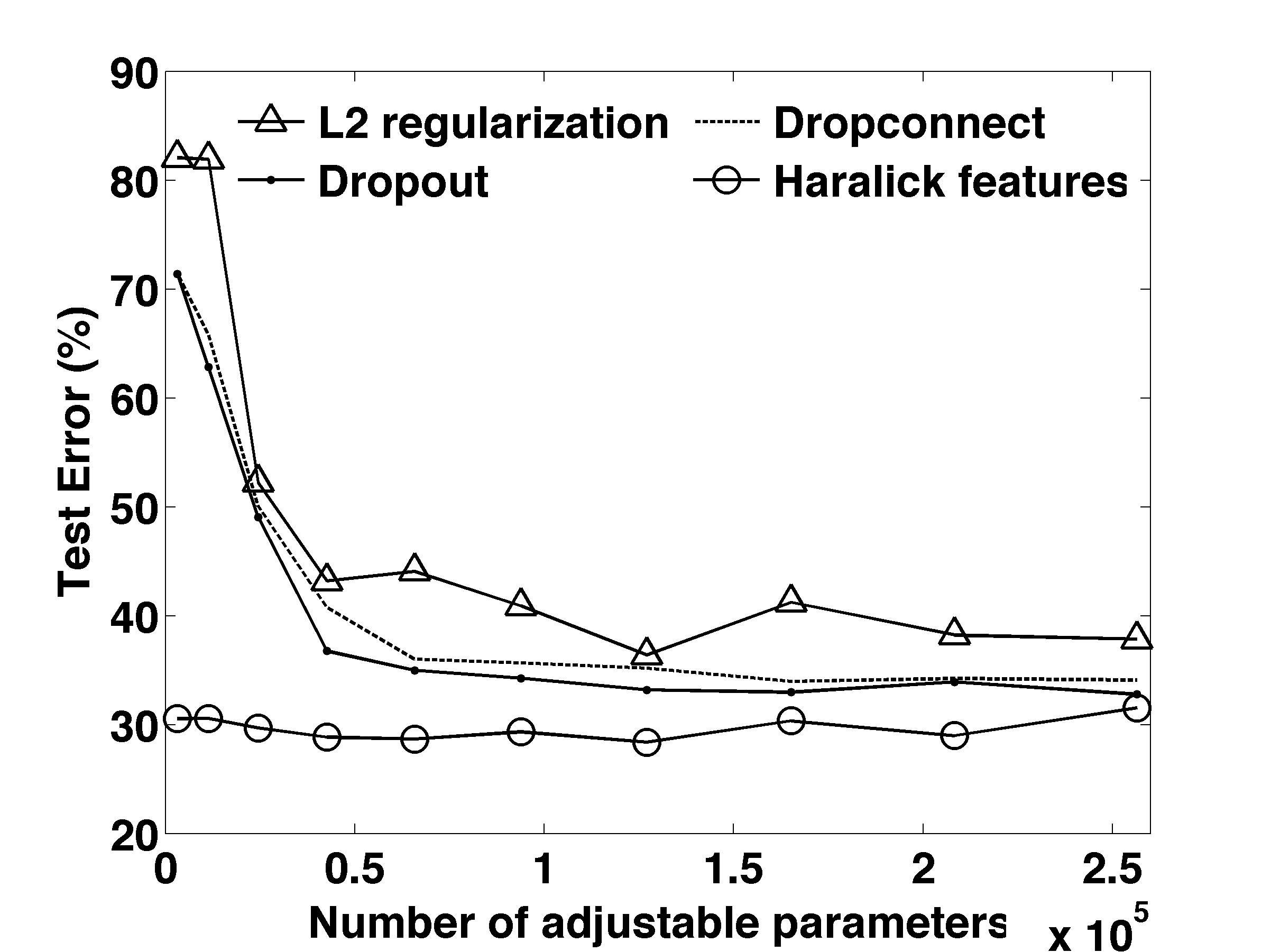}}
   \subfloat[KTH-TIPS]{\label{figur:3}\includegraphics[width=0.3\textwidth]{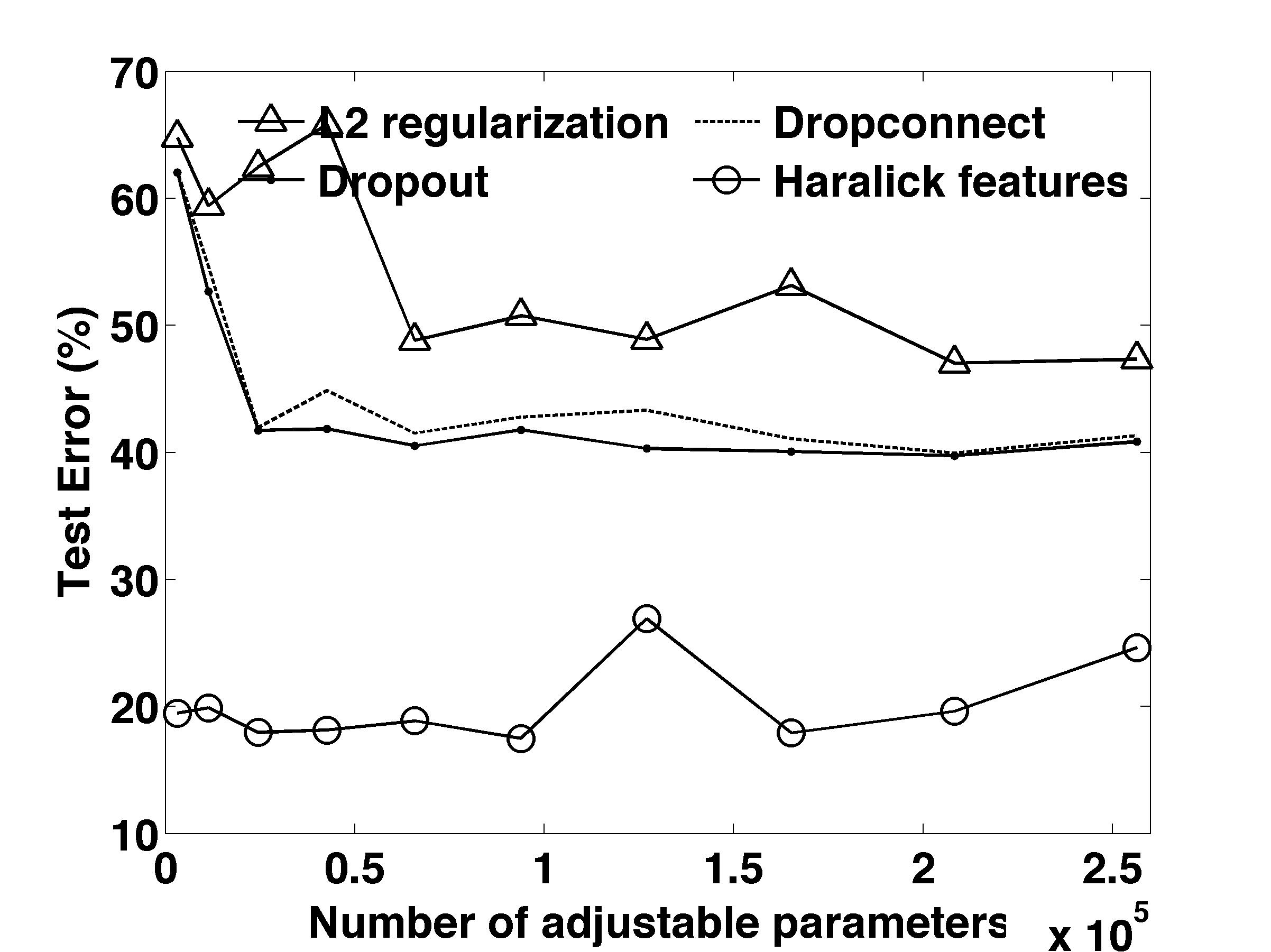}}
  \\
  \subfloat[KTH-TIPS2]{\label{figur:4}\includegraphics[width=0.3\textwidth]{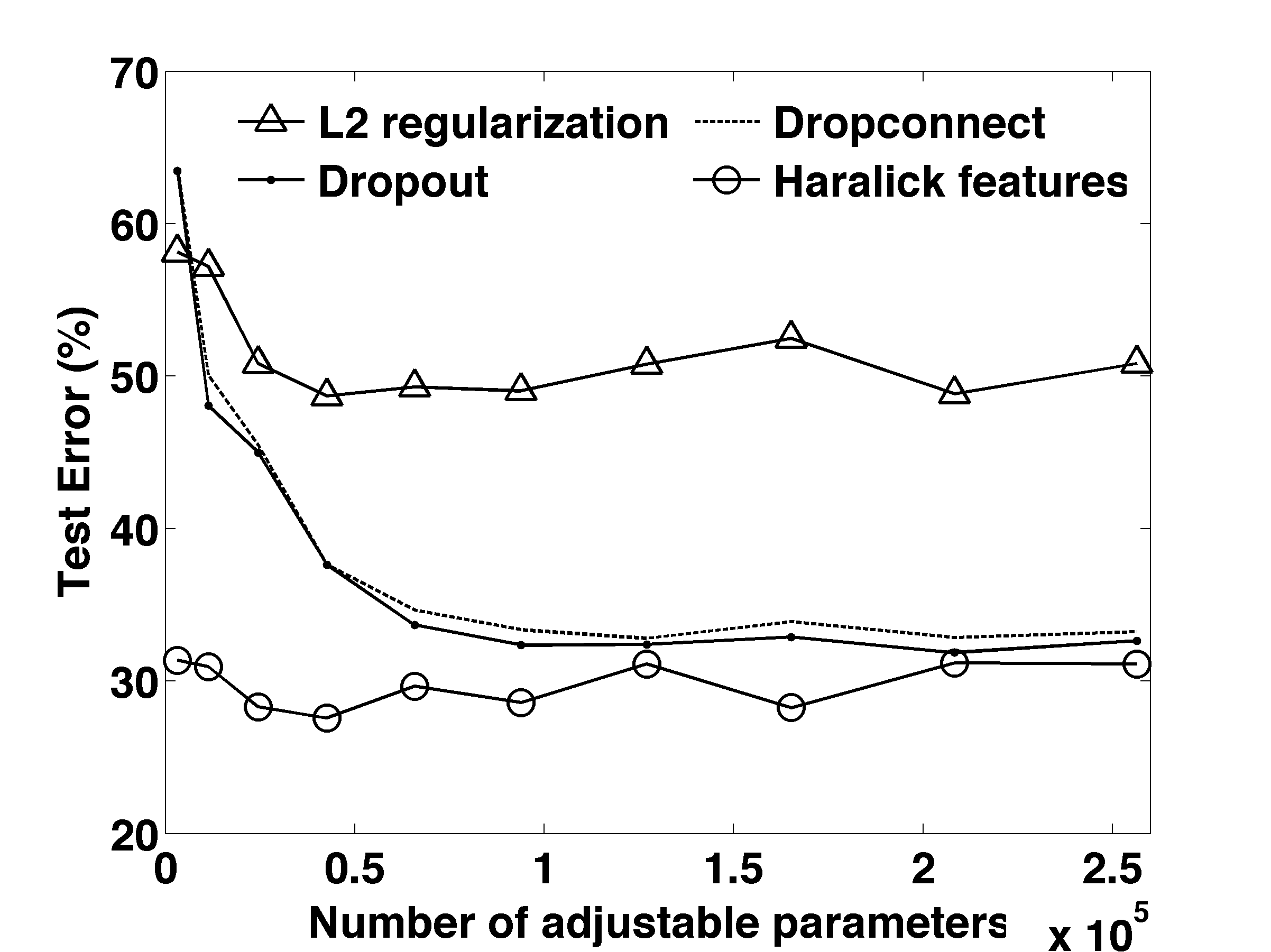}}
  \subfloat[UIUCTex]{\label{figur:5}\includegraphics[width=0.3\textwidth]{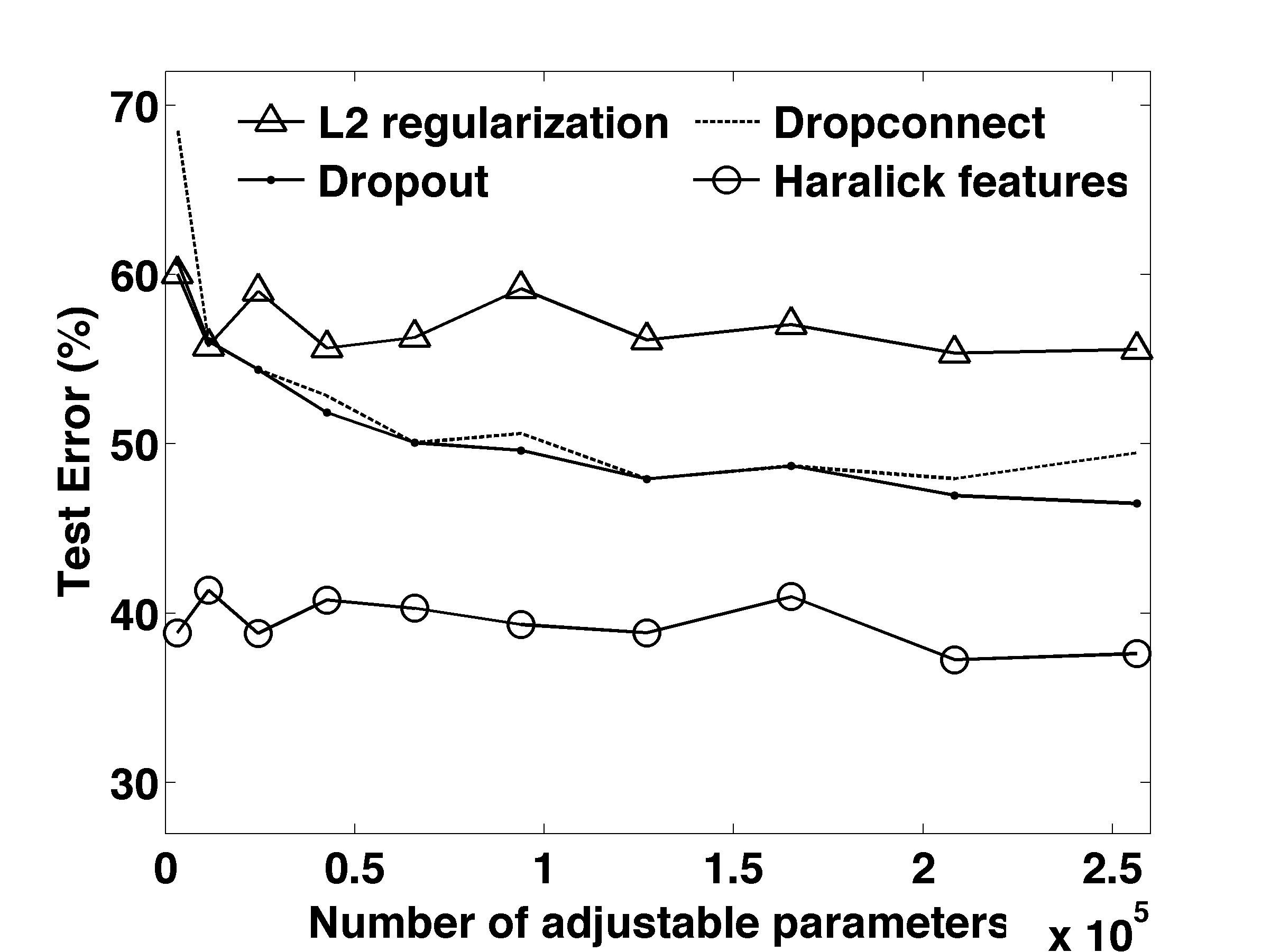}}
  \subfloat[VisTex]{\label{figur:6}\includegraphics[width=0.3\textwidth]{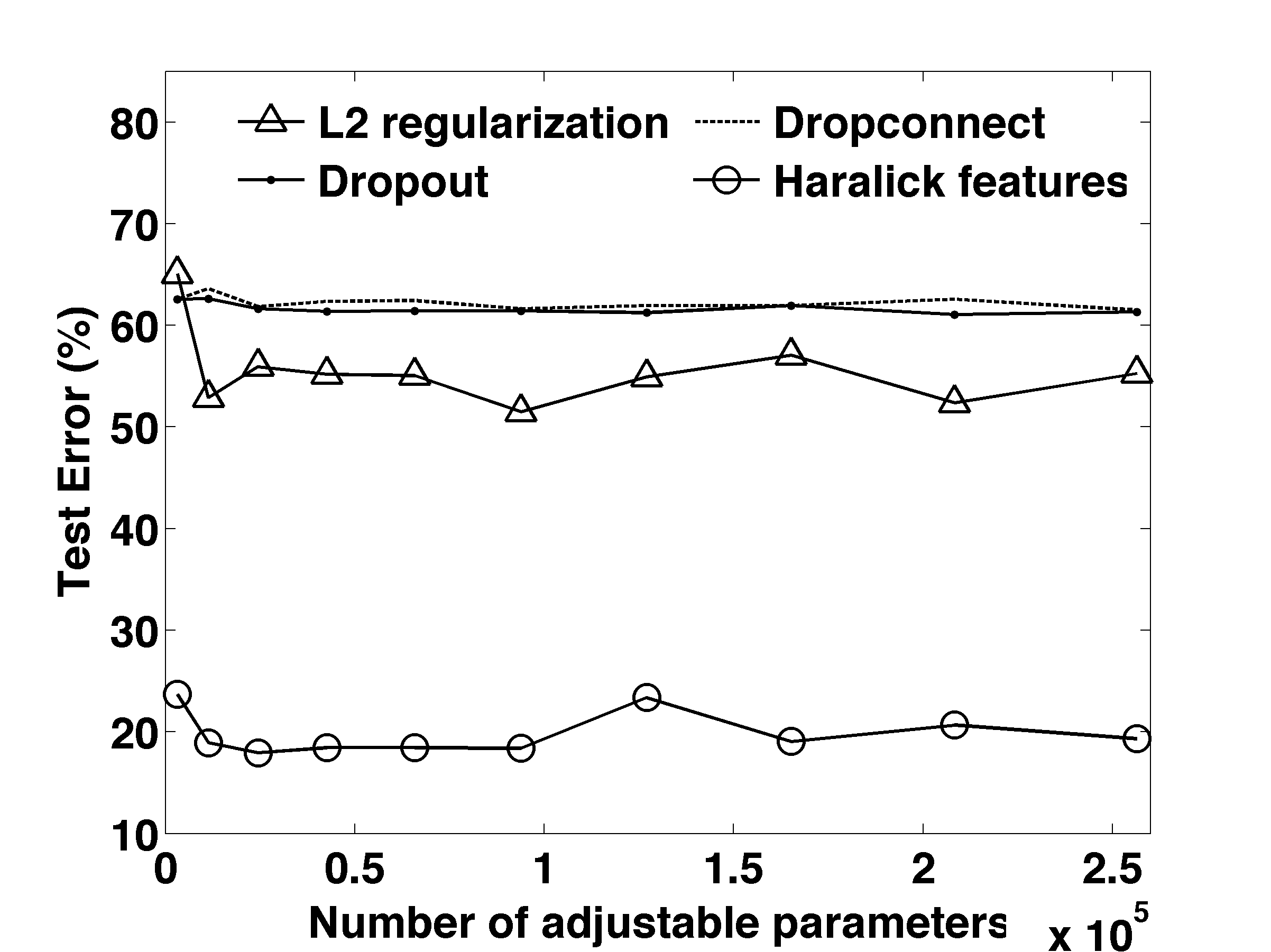}}
 \caption{Test Error on the 6 texture datasets with the Haralick features and stacked Restricted Boltzmann Machines with $L_2$ norm regularization, Dropout and Dropconnect obtained by varying the number of adjustable parameters.}\label{experiments_DBN}
\end{figure*}
\begin{figure*}[h!]
  \centering
  \subfloat[Brodatz]{\label{figur:1}\includegraphics[width=0.3\textwidth]{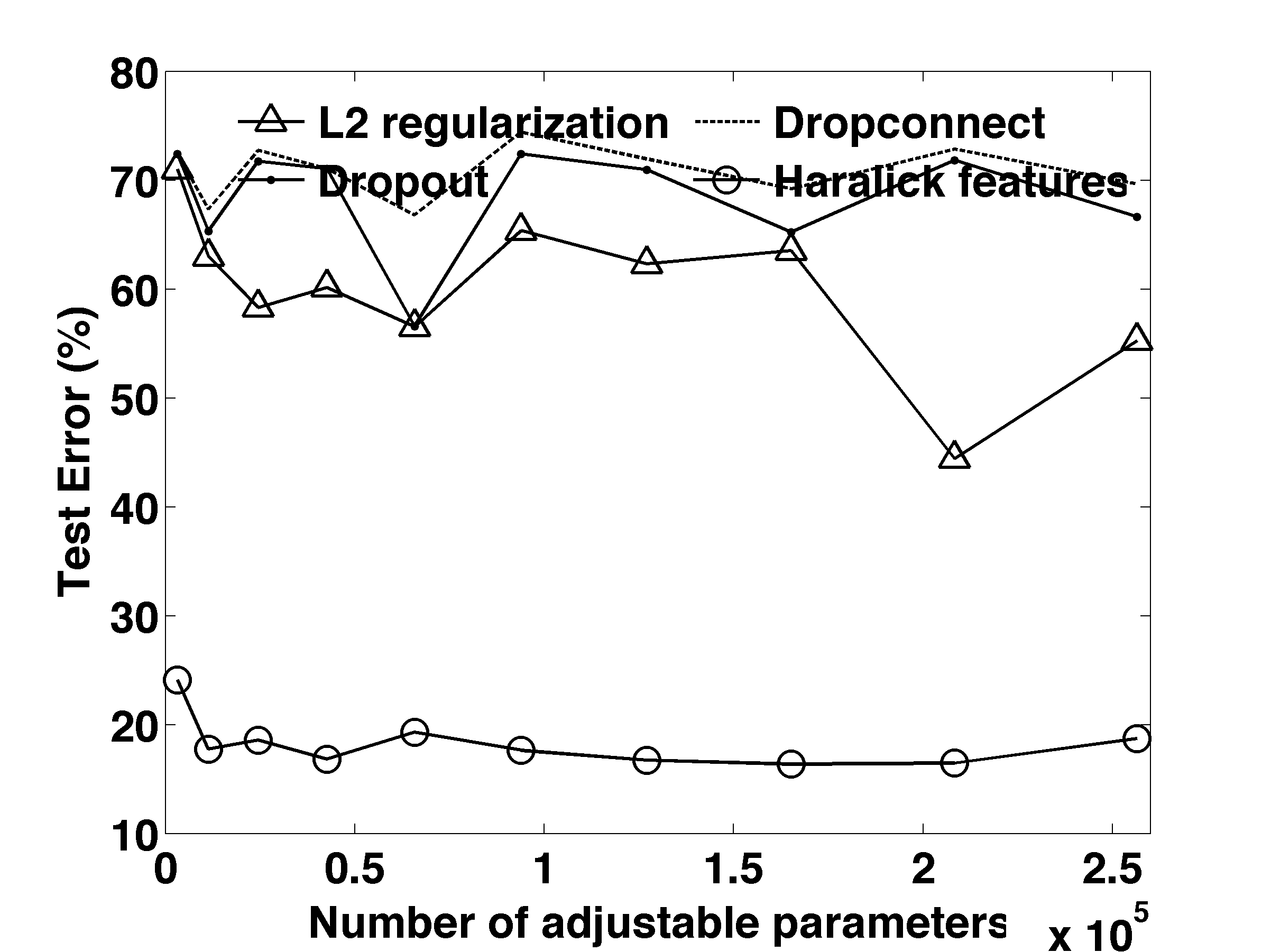}}
  \subfloat[Drexel]{\label{figur:2}\includegraphics[width=0.3\textwidth]{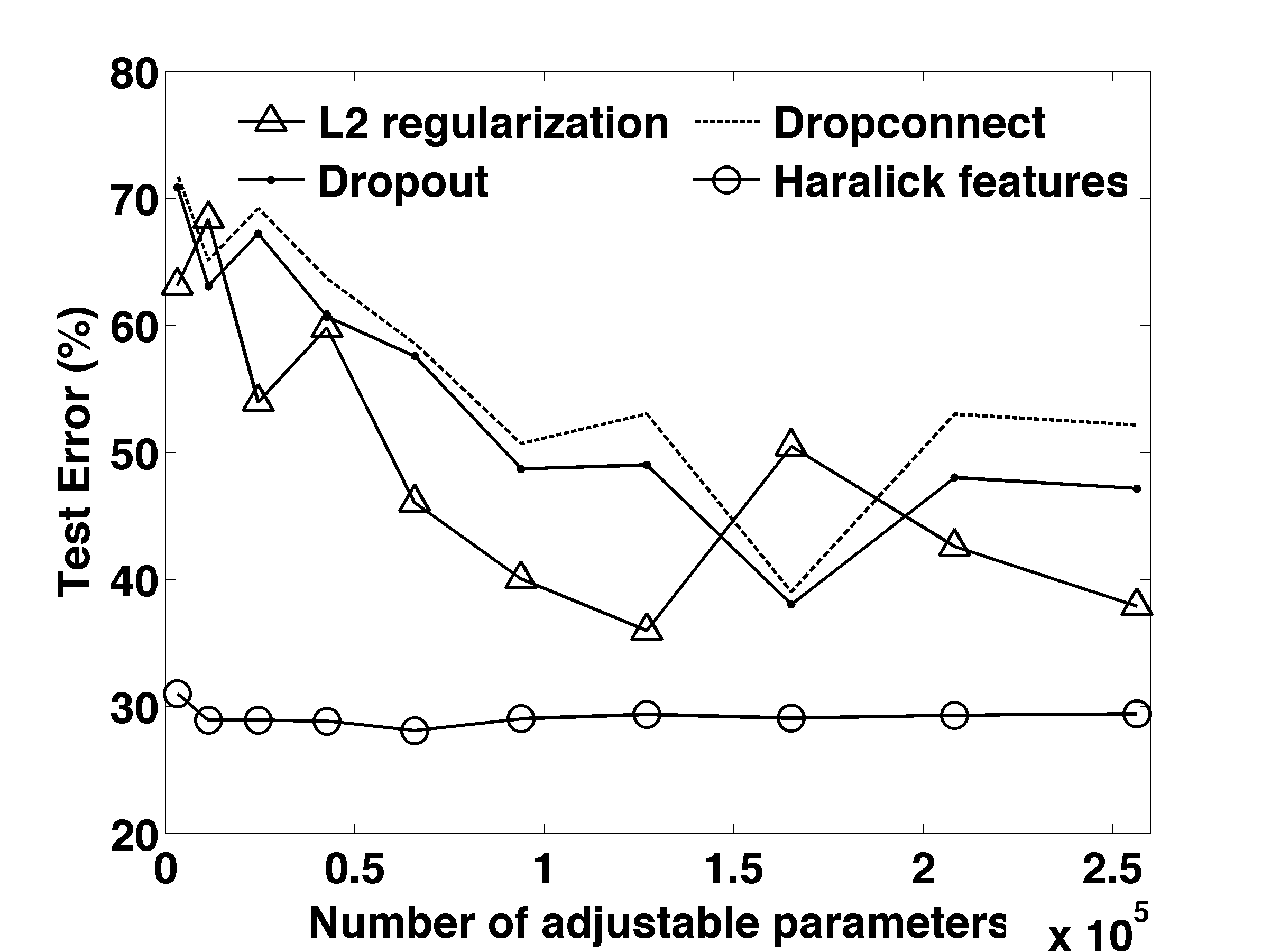}}
   \subfloat[KTH-TIPS]{\label{figur:3}\includegraphics[width=0.3\textwidth]{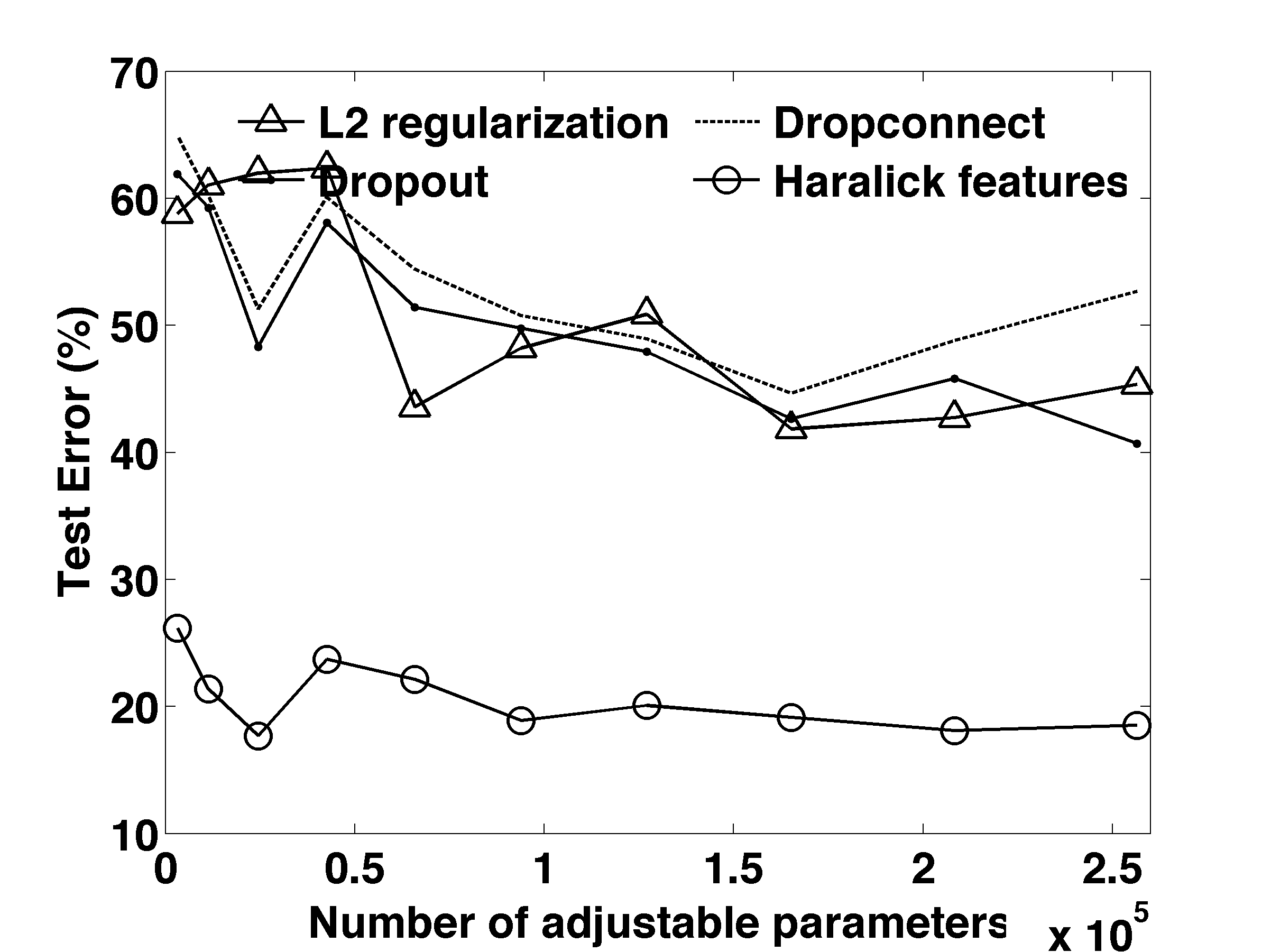}}
  \\
  \subfloat[KTH-TIPS2]{\label{figur:4}\includegraphics[width=0.3\textwidth]{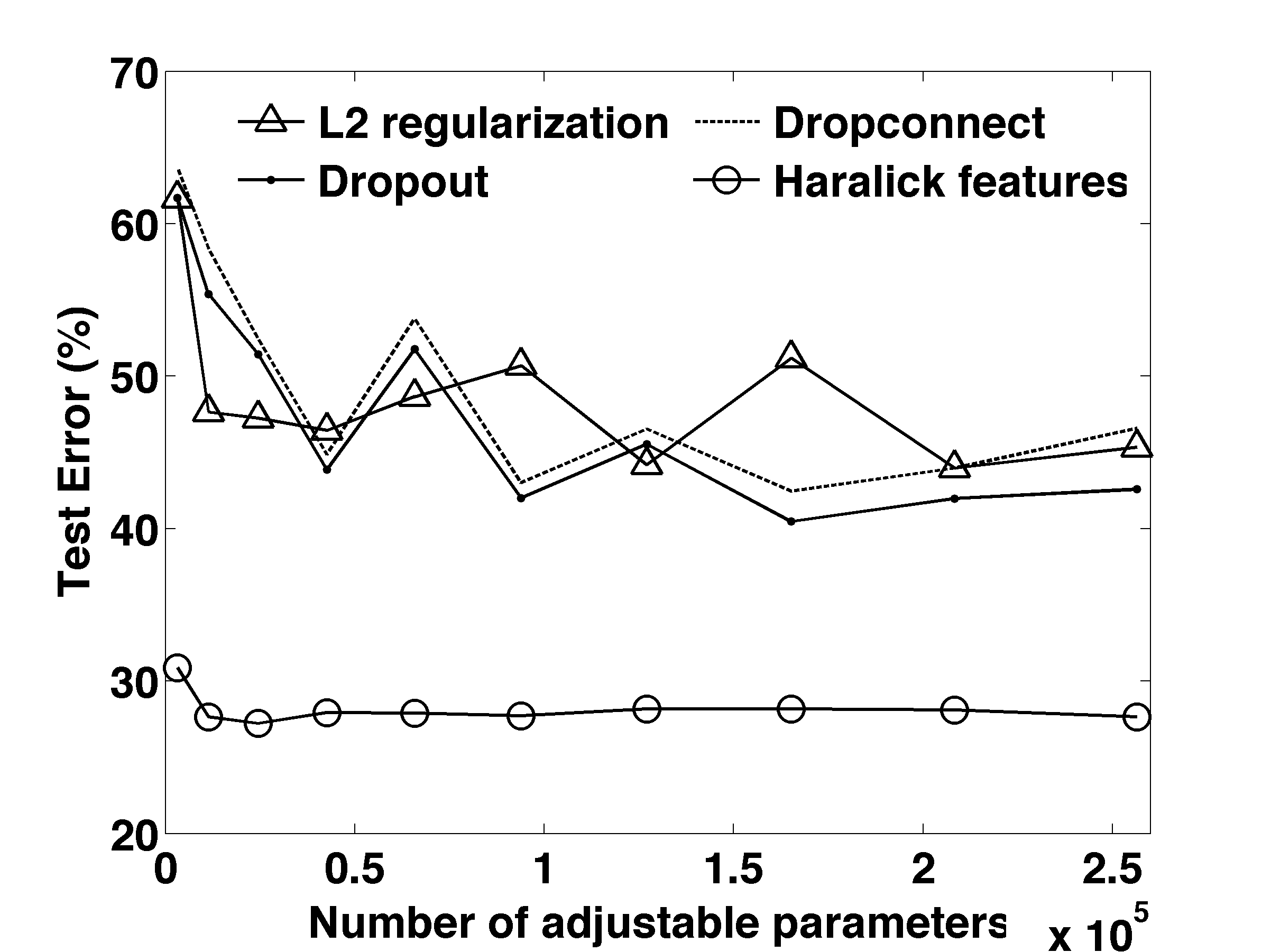}}
  \subfloat[UIUCTex]{\label{figur:5}\includegraphics[width=0.3\textwidth]{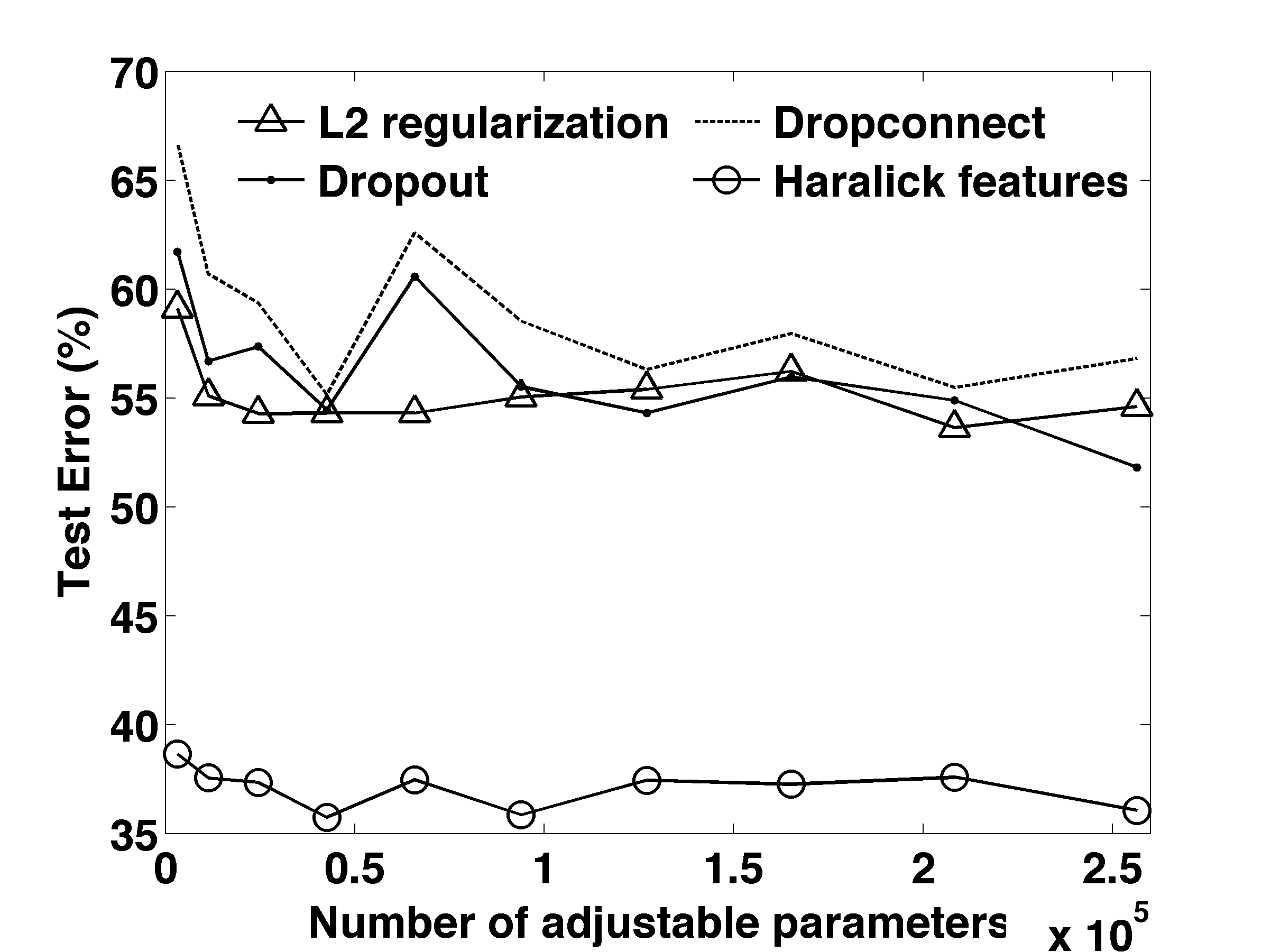}}
  \subfloat[VisTex]{\label{figur:6}\includegraphics[width=0.3\textwidth]{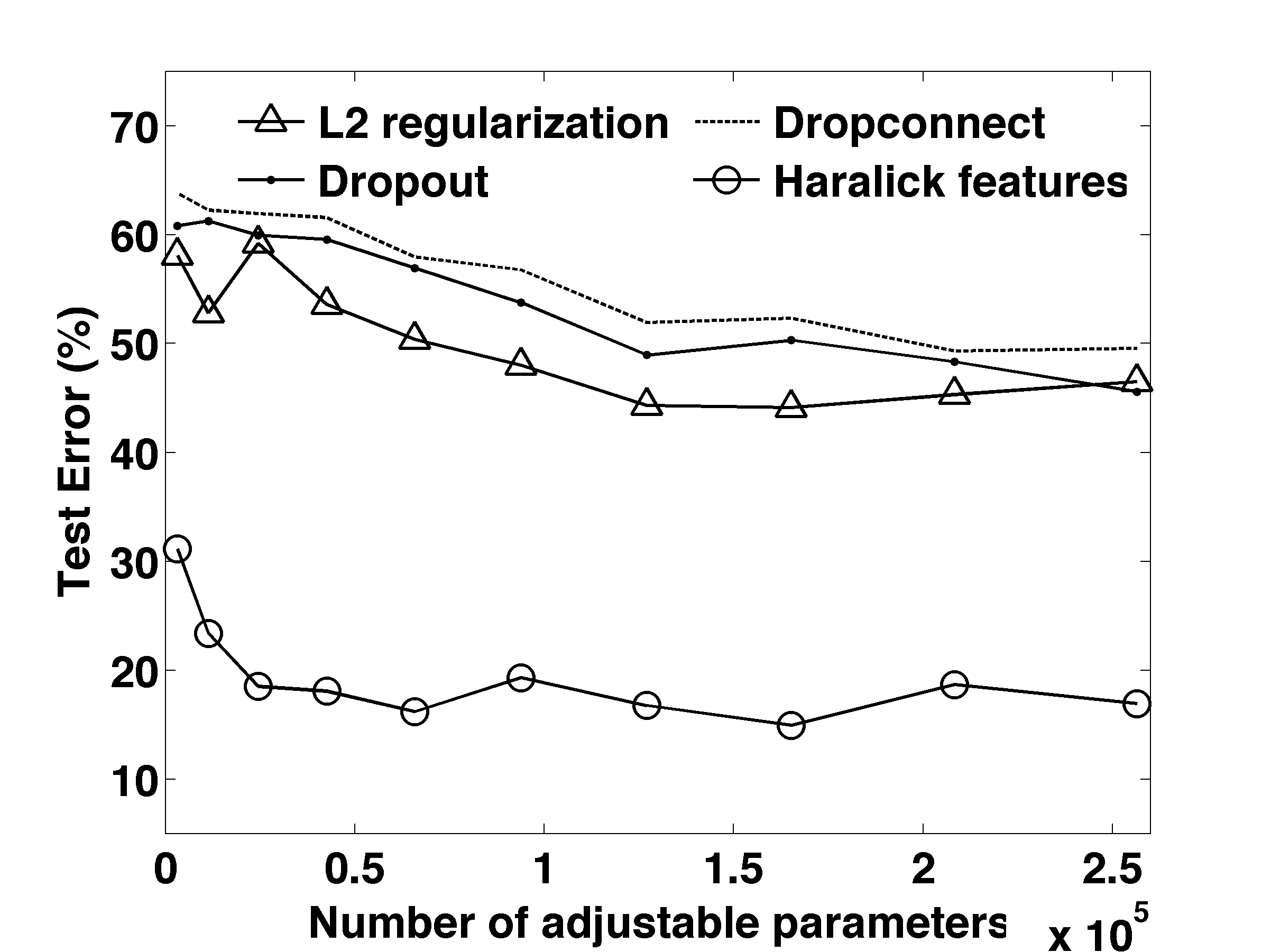}}
\caption{Test Error on the 6 texture datasets with the Haralick features and Stacked Denoising Autoencoders with $L_2$ norm regularization, Dropout and Dropconnect obtained by varying the number of adjustable parameters.} \label{experiments_SAE}
\end{figure*}  

\section{CONCLUSION}
The use of Deep Neural Networks for texture recognition has seen a significant impediment due to a lack of thorough understanding of the limitations of existing Neural architectures. In this paper, we provide theoretical bounds on the use of Deep Neural Networks for texture classification. First, using the theory of VC-dimension we establish the relevance of handcrafted feature extraction. As a corollary to this analysis, we derive for the first time upper bounds on the VC dimension of CNN as well as Dropout and Dropconnect networks and the relation between excess error rates. Then we use the concept of \emph{Intrinsic Dimension} to show that texture datasets have a higher dimensionality than color/shape based data.  Finally, we derive an important result on \emph{Relative Contrast} that generalizes the one proposed in \cite{Aggarwal01}. From the theoretical and empirical analysis, we conclude that for texture data, we need to redesign neural architectures and devise new learning algorithms that can learn GLCM or Haralick-like features from input data.
\clearpage  
\section{Acknowledgments}
This research was supported by NASA Carbon Monitoring System through Grant \#NNH14ZDA001-N-CMS and Cooperative Agreement Number NASA-NNX12AD05A, CFDA Number 43.001, for the project identified as "Ames Research Center Cooperative for Research in Earth Science and Technology (ARC-CREST)". Any opinions findings, and conclusions or recommendations expressed in this material are those of the authors and do not necessarily reflect that of NASA or the United States Government. 

\bibliographystyle{abbrv}
\bibliography{References}

\end{document}